\RequirePackage{fix-cm}
\documentclass{article}
\usepackage{iclr2027_conference,times}

\usepackage[utf8]{inputenc}
\usepackage[T1]{fontenc}
\usepackage{hyperref}
\usepackage{url}
\usepackage{amsmath,amssymb,amsthm,mathtools}
\usepackage{booktabs,multirow,array,tabularx}
\usepackage{adjustbox}
\usepackage{graphicx}
\usepackage{flafter}
\usepackage{wrapfig}
\usepackage[dvipsnames]{xcolor}
\usepackage{microtype}
\usepackage{enumitem}
\setlist[itemize]{leftmargin=*,itemsep=1.5pt,topsep=2pt,parsep=0pt}
\setlist[enumerate]{leftmargin=*,itemsep=2pt,topsep=3pt,parsep=0pt}
\usepackage[most]{tcolorbox}
\usepackage{needspace}
\usepackage{placeins}

\graphicspath{{figures/}}
\newcommand{\QualityBar}{0.02}

\newcommand{\SyncPass}{3}
\newcommand{\AsyncPass}{10}

\newcommand{\NSyncTraces}{9}

\newcommand{\DswapSync}{2.3\times 10^{-3}}
\newcommand{\DswapAsync}{1.1\times 10^{-3}}

\newcommand{\GrowMseLongF}{0.26}
\newcommand{\GrowMseLongH}{0.34}

\newcommand{\NViol}{803}
\newcommand{\NViolInDisk}{499}
\newcommand{\NRealUnforgivable}{145}

\newcommand{\CIgrH}{[-0.002,+0.037]}
\newcommand{\RecovRegF}{+85}
\newcommand{\RecovRegH}{+99}
\newcommand{\RecovGrowF}{-46}

\newcommand{\RecovPerF}{-7}
\newcommand{\RecovPerH}{+4}
\newcommand{\TileMaxAbs}{4.6\times 10^{-3}}
\newcommand{\TileRatioPersist}{1.45}
\newcommand{\TileRatioRegen}{1.197}
\newcommand{\LatticeMatched}{26}
\newcommand{\LatticeCells}{27}
\newcommand{\LatticeNonMarginal}{25}
\newcommand{\LatticeAnnWrong}{4}

\newcommand{\HorizonLong}{4{,}096}

\newcommand{\NViolRePosCplx}{159}
\newcommand{\ViolReGrow}{198}
\newcommand{\ViolReRegen}{32}
\newcommand{\ViolDiskRegen}{288}

\newcommand{\QualityMarginMin}{35}
\newcommand{\QualityMarginMax}{1{,}078}

\newcommand{\TileMedMargin}{30}
\newcommand{\OtherMseLongMax}{1.3\times 10^{-3}}
\newcommand{\GrowCarriers}{10}

\newcommand{\PersistCarriers}{10}
\newcommand{\RegenCarriers}{4}

\newcommand{\EigWindow}{32}
\newcommand{\NEigSaturated}{23}
\newcommand{\SyncFisherP}{0.003}

\newcommand{\ConeRatioSmall}{5.33}
\newcommand{\ConeFracLarge}{1.00}

\newcommand{\GateOffMaxAbs}{16.5}
\newcommand{\GateOffMedAbs}{1.05}
\newcommand{\GateOffCarrierMaxAbs}{16.5}

\newcommand{\GateOffSeedMaxAbs}{24.4}
\newcommand{\GateOffLeakMat}{6}
\newcommand{\GateOffLeakSeed}{6}
\newcommand{\GateOffLeakFactor}{8.9}

\newcommand{\FailureAliveMse}{0.0290}
\newcommand{\FailureAliveStateHF}{0.0258}
\newcommand{\FailureAliveUpdateHF}{0.0484}

\newcommand{\SwapRatioInterval}{[0.22,1.24]}
\newcommand{\QualifiedAsync}{5}
\newcommand{\QualifiedSync}{1}
\newcommand{\PositiveContrasts}{3}

\definecolor{navy}{RGB}{47,93,140}
\definecolor{brick}{RGB}{181,74,74}
\definecolor{oliveg}{RGB}{110,139,87}

\theoremstyle{plain}
\newtheorem{proposition}{Proposition}
\newtheorem{theorem}{Theorem}

\theoremstyle{remark}

\newtcolorbox{finding}{enhanced,breakable=false,colback=black!4,colframe=black!55,
  boxrule=0.5pt,arc=1.2pt,left=5pt,right=5pt,top=2.5pt,bottom=2.5pt,
  before skip=4pt,after skip=4pt}

\iclrfinalcopy
\title{Where Does Randomness Matter in\\Neural Cellular Automata?}

\author{Fei Zuo \\
Fudan University
\And
Jiaqi Shi \\
Independent Researcher
\And
Yujing Liu \\
Shanghai Pupuda Culture \\
Communication Co., Ltd.}

\begin{document}
\maketitle
\fancyhead[L]{Preprint}
\thispagestyle{fancy}

\begin{abstract}
Stochastic cell updates are often used throughout the life of a neural cellular automaton (NCA), from backpropagation through time to final rollout. This leaves two questions entangled: does update randomness help learn a useful rule, and must that randomness remain at execution? We separate training and evaluation update modes in controlled Growing NCA experiments, then vary the states shown during training. Under the standard constant-rate persist recipe, asynchronous training passes the short-horizon quality test in 10/10 runs, compared with 3/10 synchronous runs. All ten asynchronous models also retain the target for 4,096 steps under deterministic evaluation. For a scalar translation-invariant lattice, we derive an exact mean-square criterion: random masking can damp mean modes, but it also injects variance, and a mean-only test misclassifies four non-marginal settings. Finally, among 30 models that all pass the same reconstruction test, eight of ten grow-trained models become off-target at 4,096 steps, while all persist and regenerate models retain the target; damage recovery separates persist from regenerate. The results distinguish optimization reliability, execution mode, and task-specific behavior instead of treating them as one stability property.
\end{abstract}

\section{Introduction}
\label{sec:intro}

A neural cellular automaton (NCA) learns one local update rule and applies it repeatedly across a grid. The rule is optimized through finite rollouts, but after training the same parameters may be applied for thousands of steps. Random cell-update masks are often used in both stages. It is therefore easy to conflate two different claims: randomness may make optimization easier, and randomness may be required for the learned rule to work when it runs.

Prior NCA work gives reasons to use asynchronous updates, including local operation and robustness to changes in update timing \citep{niklasson2021asynchronicity}. It also shows that the training recipe changes what a learned rule can do: seed-only training can produce patterns that decay or grow, pool training encourages persistence, and damage exposure strengthens repair \citep{mordvintsev2020growing}. Synchronous training has also been reported as brittle in a multi-target self-replication setting, with successful synchronous runs producing outcomes similar to asynchronous ones \citep{sinapayen2023selfreplication}. These findings establish that update schedules and training states matter. They do not, by themselves, separate the update schedule used to learn a rule from the schedule used to execute it, or compare later-use behaviors after a common reconstruction test.

We ask two linked questions. \emph{Where does stochastic updating matter: while learning the rule, while executing it, or both?} And once a rule can reconstruct its target, \emph{which behaviors follow from the states it saw during training?} The first question requires crossing training and evaluation update modes while holding the recipe and optimizer fixed. The second requires comparing models that meet the same short-horizon quality threshold, then testing the distinct tasks of retention and repair.

\begin{figure}[t]
\centering
\includegraphics[width=\textwidth]{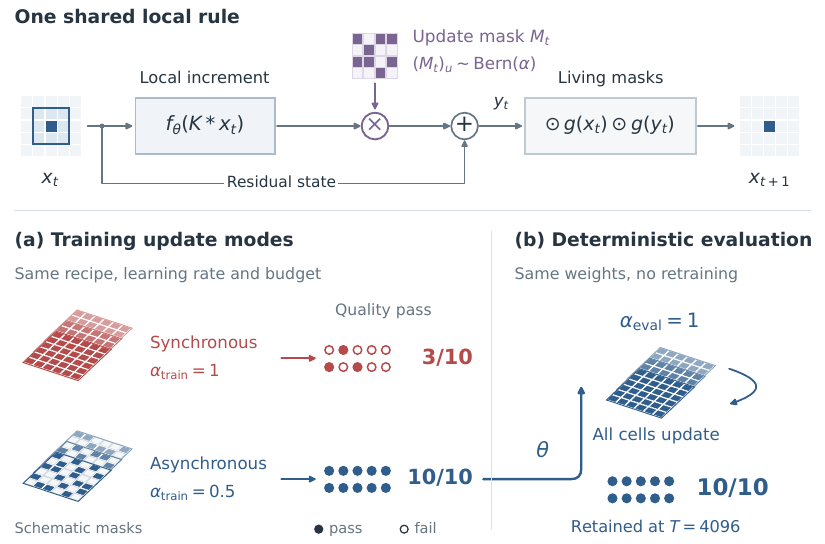}
\caption{\textbf{Random updates help obtain a usable rule under one training setting, but are not required by successful deterministic rollouts.}
The shared local rule separates the random update mask $M_t$ from the living masks.
(a) With the persist recipe, learning rate, and budget held fixed, $\SyncPass/10$ synchronous models and $\AsyncPass/10$ asynchronous models pass the quality test at $T=96$ in their training mode.
(b) With weights frozen, all ten asynchronous models retain the target through $T=\HorizonLong$ deterministic steps from the seed.
Lattices schematically depict update masks, not measured states.}
\label{fig:teaser}
\end{figure}

Our update-mode comparison separates these stages. At the tested constant learning rate, asynchronous training produces ten quality-passing models, compared with three under synchronous training. Yet every asynchronous model retains the target through 4,096 deterministic steps without retraining (Fig.~\ref{fig:teaser}). Four exploratory lower-rate synchronous retrainings also pass, so the observed training difference is specific to the tested optimizer setting. Earlier work already reports synchronous-training brittleness in another NCA task; our result is a controlled train/evaluation crossover on the canonical persistent Growing NCA, not a claim that the failure itself is new.

The update mask also has a two-sided effect in a tractable linear system. We derive an exact second-moment recursion for a scalar translation-invariant lattice with independent Bernoulli cell masks. It shows why averaging the mask can give the wrong answer: random updating changes the mean dynamics and injects state variance. The theorem is a boundary on what can be inferred from a fixed linear rule; it does not explain the nonlinear training failures.

Finally, we compare three established training recipes after all models pass the same short-horizon reconstruction test. Grow, persist, and regenerate differ in which states the rule encounters during optimization. Their long-horizon retention and damage recovery differ accordingly. These measurements make a practical distinction: update randomness affects whether a rule is learned under the tested optimizer, while state exposure affects which later-use tasks that rule supports.

This study contributes:
\begin{enumerate}
\item A controlled separation of training-time and evaluation-time update schedules. Asynchronous training improves final quality-pass rate under one fixed-rate persist recipe, while successful asynchronous models can be evaluated deterministically.
\item An exact mean-square decay criterion for a scalar linear lattice with independent cell masks, showing when mean contraction is insufficient.
\item A matched-quality comparison of retention and repair across established training-state recipes, with long rollouts and damage tests used as task outcomes rather than as one aggregate stability score.
\end{enumerate}

\section{Model and experimental design}
\label{sec:setup}

\textbf{The update rule.} We use the canonical Growing NCA \citep{mordvintsev2020growing}. Each cell stores four RGBA channels and twelve hidden channels. Fixed identity and Sobel filters form a 48-channel local perception $K*x$. A shared two-layer $1\times1$ network $f_\theta$, with 128 hidden units, ReLU activation, and a zero-initialized output layer, predicts an additive state change. A step is
\begin{equation}
\label{eq:nca}
\begin{aligned}
y_t &= x_t+M_t\odot f_\theta(K*x_t),\\
x_{t+1} &= g(x_t)\odot g(y_t)\odot y_t,
\qquad (M_t)_u\sim\mathrm{Bern}(\alpha).
\end{aligned}
\end{equation}
The Bernoulli update masks are independent across cells and steps and shared across the channels within a cell. The living mask $g$ is one when the maximum alpha value in a $3\times3$ neighborhood exceeds $0.1$; it is applied both before and after the candidate update. An entirely dead state is absorbing. We use \emph{update mask} for $M_t$ and \emph{living mask} for $g$ because they play different roles.

\textbf{Two training comparisons.} The targets are flower and heart images on $72\times72$ grids, with five seeds per configuration. Base runs use Adam, constant learning rate $2\cdot10^{-3}$, batch size eight, 8,000 optimization steps, rollout lengths sampled uniformly from 64 to 96, and the same overflow penalty. The update-mode comparison contains twenty persist models: ten synchronous ($\alpha_{\mathrm{train}}=1$) and ten asynchronous ($\alpha_{\mathrm{train}}=0.5$). Each model is evaluated in both update modes. The recipe comparison contains thirty asynchronous models, ten each trained with grow, persist, or regenerate. The ten asynchronous persist models are shared between comparisons, giving forty distinct base models. Four lower-rate synchronous retrainings are exploratory.

Grow starts every rollout from the seed. Persist samples from a pool of 1,024 generated states and replaces the worst sample with the seed. Regenerate adds circular damage to half of the sampled batch. These recipes were introduced in the original Growing NCA work; our comparison controls the architecture, objective, optimizer, training budget, target set, and update probability while changing the training-state distribution.

\textbf{Three behavioral requirements.} A model passes the short-horizon quality test if its RGBA MSE is below $0.02$ at $T=96$ in its training update mode. For long rollouts, an empty state is labeled dead; a non-finite state or one with maximum absolute value above $10^3$ is labeled divergent; and a surviving finite state with final MSE above $0.15$ is off-target. The remaining category is called retained. The long-horizon threshold is deliberately distinct from the stricter reconstruction threshold.

At $T=96$, we also measure finite-time perturbation growth by adding $\delta_0=10^{-3}\mathcal N(0,I)$ on living cells and evolving perturbed and clean states for 64 steps with the same update masks:
\begin{equation}
\label{eq:pert}
\lambda_{\mathrm{pert}}=\frac{1}{64}\log\frac{\|\delta_{64}\|_2}{\|\delta_0\|_2}.
\end{equation}
This is a finite-amplitude trajectory measurement, not a maximum Lyapunov exponent. Damage recovery is evaluated separately over a fixed 80-step window. The full protocols, statistical units, and outcome thresholds are in App.~\ref{app:instruments}.

\section{Random updates change learning success more than execution mode}
\label{sec:noise}

To isolate the update schedule, we vary $\alpha_{\mathrm{train}}$ while fixing the persist recipe and optimizer. We then freeze each model's parameters and evaluate it using both synchronous and asynchronous updates. This design separates the probability of obtaining a usable final rule from a requirement on how that rule must run.

\textbf{Final training success differs under the tested optimizer.} All ten asynchronous models pass the $T=96$ quality test, compared with $\SyncPass$ of ten synchronous models (Table~\ref{tab:factorial}); the two-sided Fisher exact test gives $p=\SyncFisherP$. This is a controlled comparison of two update schedules under one learning rate and one training recipe, not a general statement that synchronous NCA training cannot succeed.

The available synchronous training traces add context but do not replace final evaluation. All $\NSyncTraces$ recorded traces cross a minibatch loss of $10^{-3}$, while seven of ten final synchronous checkpoints fail the quality test. Six failures are empty states and one is alive but deformed; five available traces end near the empty-canvas loss. Because the traces are incomplete and the loss is measured on a training minibatch, they show that low observed loss can occur during optimization, not that every failed run previously held a robust rule. An exploratory single-run checkpoint replay is documented separately in App.~\ref{app:probes}; it illustrates one trajectory and does not estimate how often intermediate rules remain usable.

\textbf{Successful asynchronous models do not need stochastic evaluation.} All ten asynchronous models retain the target for 4,096 deterministic steps from the seed. Among quality-passing models, switching the evaluation update mode changes median absolute MSE by $\DswapAsync$ for asynchronous-trained models and $\DswapSync$ for synchronous-trained models at $T=1024$. The preregistered asymmetry comparison did not establish a directional swap effect. The useful conclusion is narrower: in this cohort, every asynchronous-trained model that passes short-horizon quality can be run deterministically and remains in the retained class.

Four exploratory retrainings of two failed flower configurations at learning rates $10^{-3}$ and $5\cdot10^{-4}$ all pass. This control limits the interpretation of the main comparison: asynchronous updating improves the pass rate at the tested constant learning rate; it is not necessary for synchronous training to succeed.

\begin{table}[!t]
\centering
\caption{\textbf{Training and evaluation update modes.} Quality is tested at $T=96$ in the training mode. Retention is the bounded, on-target class after $\HorizonLong$ deterministic steps. Swap cost is the median absolute MSE change at $T=1024$ among quality-passing models. Each target has five models per training mode.}
\label{tab:factorial}
\small
% Generated from immutable results.
\setlength{\tabcolsep}{4pt}
\renewcommand{\arraystretch}{1.18}
\begin{tabular*}{\textwidth}{@{\extracolsep{\fill}}llccc@{}}
\toprule
Training & Target & Quality pass & Retained (det.) & \shortstack{Swap difference\\$|\Delta\mathrm{MSE}|$} \\
\midrule
Synchronous ($\alpha=1$) & Flower & 1/5 & 1/5 & $2.3\times 10^{-3}$ \\
 & Heart & 2/5 & 3/5 & $2.8\times 10^{-3}$ \\
\addlinespace[4pt]
Asynchronous ($\alpha=0.5$) & Flower & 5/5 & 5/5 & $2.5\times 10^{-3}$ \\
 & Heart & 5/5 & 5/5 & $1.0\times 10^{-3}$ \\
\bottomrule
\end{tabular*}

\end{table}

\section{A random mask damps some mean modes and injects variance}
\label{sec:theory}

The NCA comparison shows when the update mask affects learning, but it does not identify a mechanism for the training failures. We therefore analyze a fixed linear residual rule, where the effect of independently selecting cells can be computed exactly. Consider
\[
x_{t+1}=x_t+M_tAx_t,
\]
where $A$ is translation invariant on a periodic scalar lattice and each diagonal entry of $M_t$ is an independent Bernoulli variable with mean $\alpha$.

\begin{proposition}[Mean dynamics]
\label{prop:disk}
For $0<\alpha\leq1$, $\mathbb E[x_{t+1}]=(I+\alpha A)\mathbb E[x_t]$. If $\mu$ is an eigenvalue of the synchronous map $I+A$, then the corresponding mean multiplier lies in
\[
D_\alpha=\{\mu\in\mathbb C:|\mu-(1-1/\alpha)|<1/\alpha\}.
\]
\end{proposition}

The disk $D_\alpha$ passes through $+1$; at $\alpha=1$ it is the unit disk, while as $\alpha$ decreases its left edge moves outward. In particular, for a unit-modulus mode $\mu=e^{i\theta}$,
\[
|(1-\alpha)+\alpha e^{i\theta}|^2
=1-2\alpha(1-\alpha)(1-\cos\theta).
\]
Partial updates damp non-neutral oscillatory modes in the mean. This statement concerns the expected state, not the energy of a randomly updated trajectory.

\begin{theorem}[Exact second-moment dynamics]
\label{thm:ms}
Let $x_t$ be a scalar field on the $N^d$ periodic lattice, and let $A$ have Fourier symbol $\hat a(\omega)$. With masks independent across cells and time, the expected modal power $P_t(\omega)=\mathbb E|\hat x_t(\omega)|^2$ satisfies
\begin{equation}
\label{eq:ms}
P_{t+1}(\omega)=d_\alpha(\omega)P_t(\omega)
+\frac{\alpha(1-\alpha)}{N^d}
 \sum_{\omega'}|\hat a(\omega')|^2P_t(\omega'),
\end{equation}
where $d_\alpha(\omega)=|1+\alpha\hat a(\omega)|^2$. All non-neutral modes decay in mean square for every finite-second-moment initial state if and only if $d_\alpha(\omega)<1$ on those modes and
\begin{equation}
\label{eq:criterion}
S=\sum_{\omega:\hat a(\omega)\ne0}
\frac{\alpha(1-\alpha)|\hat a(\omega)|^2}
{N^d[1-d_\alpha(\omega)]}<1.
\end{equation}
Neutral modes are excluded from this decay criterion.
\end{theorem}

The recursion is a diagonal modal update plus a rank-one power injection. The first term propagates each mode under the averaged rule; the second is the variance introduced by the random mask and couples the modal powers. Thus, replacing the mask by its mean can overstate stability. The criterion makes that discrepancy explicit through $S$. Proofs, including the neutral-mode case, are in App.~\ref{app:proofs}.

We test the criterion on a $32\times32$ diffusion lattice with update $x\leftarrow x+M(c\,\mathrm{Lap}\,x)$. The second-moment test matches all $\LatticeNonMarginal$ non-marginal simulation labels and resolves $\LatticeAnnWrong$ settings that the mean-only test calls stable. The complete grid contains $\LatticeCells$ configurations; the two exactly marginal cases are reported separately because finite-time labels there need not match asymptotic decay (Fig.~\ref{fig:theory}, App.~\ref{app:lattice}). This fixed-rule result limits the interpretation of update randomness: it can suppress selected mean modes while still increasing trajectory variance. It is not a derivation of the nonlinear training outcome in Sec.~\ref{sec:noise}.

\begin{figure}[!t]
\centering
\includegraphics[width=\textwidth]{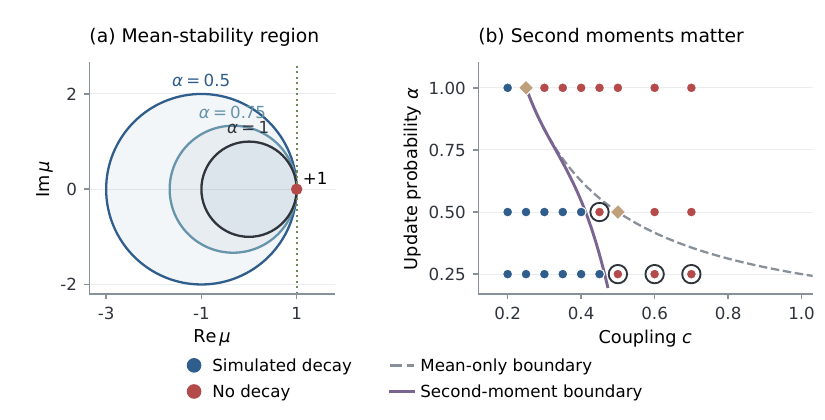}
\caption{\textbf{Mean contraction and mean-square decay are different tests.}
(a) Mean-stability disks for different update rates.
(b) Simulated outcomes on the diffusion lattice. The exact second-moment criterion separates the non-marginal stable and unstable cases; rings mark the mean-only disagreements and diamonds mark the two marginal configurations.}
\label{fig:theory}
\end{figure}

\section{Training-state exposure determines retention and repair}
\label{sec:regimes}

Update timing is only one training choice. The distribution of states used by the loss determines which parts of the learned trajectory receive direct pressure. To compare these effects, we evaluate grow, persist, and regenerate models that all pass the same $T=96$ quality threshold, then test retention through 4,096 steps and recovery after damage.

\textbf{Equal reconstruction does not imply equal later behavior.} All thirty recipe models pass the short-horizon quality test, with group median errors from $\QualityMarginMin$ to $\QualityMarginMax$ times below the threshold. By $T=4096$, eight of ten grow models are off-target, whereas all ten persist and all ten regenerate models remain in the retained class (Table~\ref{tab:regimes}). These outcomes are measured after training and do not follow from the pass threshold alone.

The trajectory perturbation measurement provides a complementary, but not interchangeable, description. Persist models have negative mean $\lambda_{\mathrm{pert}}$ on both targets, while grow models have positive mean values. The grow-minus-persist bootstrap intervals exclude zero on both targets. Regenerate models span positive and negative responses; their group intervals include zero on both targets. Positive perturbation growth is not itself a failure label: five models with positive $\lambda_{\mathrm{pert}}$ still retain the pattern. We therefore report the direct long-horizon task outcome alongside the trajectory measurement (Fig.~\ref{fig:regimes}).

\textbf{Damage exposure adds a distinct behavior.} Under the same damage protocol, regenerate models reduce damage-induced error by a mean of $\RecovRegF\%$ on flower and $\RecovRegH\%$ on heart. Persist changes error little on average ($\RecovPerF\%$ and $\RecovPerH\%$), while grow worsens it ($\RecovGrowF\%$ on each target). The grow-minus-regenerate interval on heart includes zero, so the target-specific contrast is not equally strong everywhere. These results are consistent with the training inputs: pool states expose models to formed patterns, while circular damage directly trains recovery. The recipes themselves are established; our contribution is the controlled comparison of their task outcomes under one architecture and protocol.

\begin{table}[!t]
\centering
\caption{\textbf{Models with similar reconstruction quality learn different later-use behavior.} MSE columns are medians over five models per target. Perturbation growth is the model mean, with bootstrap 95\% intervals below. Recovery is the mean percentage under the frozen damage protocol. Retention uses the long-horizon threshold, not the stricter $T=96$ quality threshold.}
\label{tab:regimes}
\small
% Generated from immutable results.
{\fontsize{8.5}{10.5}\selectfont
\setlength{\tabcolsep}{4pt}
\renewcommand{\arraystretch}{1.18}
\begin{tabular*}{\textwidth}{@{\extracolsep{\fill}}llccccc@{}}
\toprule
Recipe & Target & \shortstack{MSE\\$T=96$} & \shortstack{$\lambda_{\mathrm{pert}}$\\95\% CI} & \shortstack{MSE\\$T=4096$} & \shortstack{Recovery\\(\%)} & Retained \\
\midrule
Grow & Flower & $6.6\times 10^{-5}$ & \shortstack{$+0.017$\\$[+0.011,+0.021]$} & $0.26$ & $-46$ & 1/5 \\
 & Heart & $1.9\times 10^{-5}$ & \shortstack{$+0.023$\\$[+0.017,+0.028]$} & $0.34$ & $-46$ & 1/5 \\
\addlinespace[4pt]
Persist & Flower & $1.1\times 10^{-4}$ & \shortstack{$-0.024$\\$[-0.031,-0.015]$} & $6.5\times 10^{-4}$ & $-7$ & 5/5 \\
 & Heart & $6.1\times 10^{-5}$ & \shortstack{$-0.033$\\$[-0.040,-0.027]$} & $8.5\times 10^{-5}$ & $+4$ & 5/5 \\
\addlinespace[4pt]
Regenerate & Flower & $4.3\times 10^{-4}$ & \shortstack{$-0.008$\\$[-0.017,+0.006]$} & $7.0\times 10^{-4}$ & $+85$ & 5/5 \\
 & Heart & $5.7\times 10^{-4}$ & \shortstack{$+0.002$\\$[-0.015,+0.025]$} & $1.3\times 10^{-3}$ & $+99$ & 5/5 \\
\bottomrule
\end{tabular*}
}

\end{table}

\begin{figure}[!t]
\centering
\includegraphics[width=\textwidth]{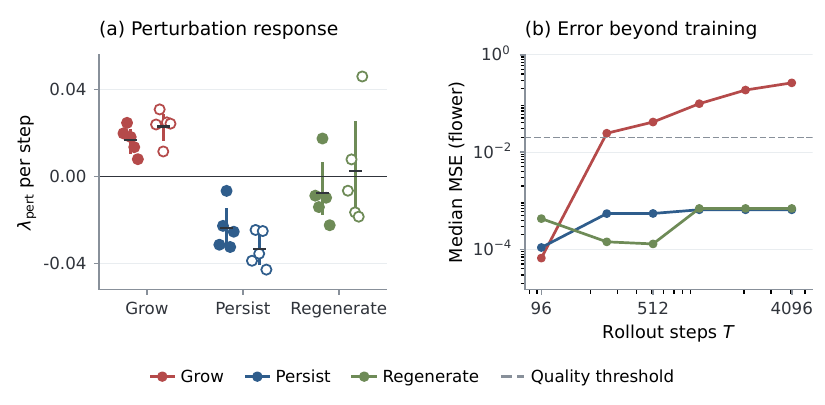}
\caption{\textbf{State exposure changes trajectory response and long-horizon error.}
(a) Individual finite-time perturbation growth and group intervals.
(b) Flower median MSE as rollout length increases beyond the maximum training horizon; heart outcomes are summarized in Table~\ref{tab:regimes}.}
\label{fig:regimes}
\end{figure}

\textbf{A frozen local spectrum does not substitute for these tests.} All thirty recipe models have a largest returned frozen-state Jacobian eigenvalue above one, yet persist and regenerate models retain the target while most grow models become off-target. This measurement does not separate the observed fates. Removing the living masks leaves all evaluated rollouts numerically bounded, while two regenerate models change target class. The full eigenvalue counts and mask-removal controls are reported in App.~\ref{app:eigs} and App.~\ref{app:gateoff}. These controls reinforce the need to specify which property is being tested: local expansion, numerical boundedness, target retention, and damage recovery are not interchangeable.

\section{Discussion}
\label{sec:design}

The results support a simple design sequence for a learned local rule. First, choose an update schedule that gives reliable optimization under the available budget. In the tested fixed-rate persist setup, asynchronous updates improve the final pass rate; the lower-rate synchronous probes show that this is an optimizer-dependent advantage. Second, test whether the learned model requires stochastic updates at execution. In our cohort, every passing asynchronous model remains on-target under deterministic evaluation. Third, train on the states that the deployed rule must handle: generated states support retention, and explicit damage exposure supports repair. Reconstruction loss, local spectrum, and finite-time perturbation growth each answer only part of this sequence.

The analytical result clarifies one limit of this advice. Independent cell masks change expected linear dynamics and add a variance term, so random updating is not a universal stability switch. The exact criterion applies only to a fixed scalar translation-invariant linear operator. It does not predict the absorbing-state failures of the nonlinear NCA or replace direct evaluation of the trained rule.

The scope is intentionally narrow: two synthetic targets, one Growing NCA family, five seeds per target and recipe, one primary learning rate, and a finite rollout horizon. Synchronous-training failures have prior precedent, and our lower-rate probes are exploratory. The results establish a controlled distinction for this model and protocol; broader claims require other targets, architectures, and optimizer schedules.

\section{Related work}
\label{sec:related}

\textbf{NCA training and applications.} Growing NCA establishes seed growth, pool-based persistence, and damage-based regeneration \citep{mordvintsev2020growing}. Subsequent work explores asynchronous update timing, identity mappings, and applications across dynamic textures, classification, graph and generative modeling, and 3D morphogenesis \citep{niklasson2021asynchronicity,stovold2025identity,pajouheshgar2023dynca,randazzo2020selfclassifying,grattarola2021gnca,palm2022vnca,tesfaldet2022vitca,sudhakaran2021minecraft,najarro2022hypernca,mordvintsev2022isotropic}. Recent reviews synthesize NCA architectures and reference implementations \citep{ncareview2026}. Synchronous training has also been reported as brittle in multi-target self-replication \citep{sinapayen2023selfreplication}.

\textbf{Attractor geometry and dynamical analysis.} Recent work examines deterministic NCA attractors, perturbation responses, and transient dynamics \citep{attractors2026,masumori2026selfmaintenance,transient2026}. We complement this perspective through controlled train/evaluation crossovers and common-quality recipe comparisons under the canonical Growing NCA setting. Broader dynamical foundations span classic cellular automata, criticality, morphogenetic pattern formation, and continuous artificial life \citep{vonneumann1966,gardner1970life,wolfram1984universality,langton1990edge,bak1987soc,turing1952chemical,chan2019lenia}, as well as reservoir computing with evolved critical NCA \citep{pontesfilho2025reservoir}.

\textbf{Stochastic updates, neural dynamics, and stability.} Multiplicative noise, dropout, and stochastic depth regularize recurrent and deep networks \citep{jim1996noise,lim2021noisyrnn,srivastava2014dropout,huang2016stochasticdepth,bertsekas1989parallel,combettes2015stochastic}. We use classical stochastic stability tools \citep{khasminskii2012stochastic,horn2013matrix} to derive exact second-moment lattice recursions. Residual dynamics, signal propagation, and nonnormal spectra further govern information flow and transient growth \citep{haber2017stable,poole2016exponential,schoenholz2017deep,trefethen2005spectra}. Connections between cellular automata and neural networks provide architectural context \citep{gilpin2019cellular,springer2021hard}. Optimization and recurrent dynamics in related implicit and continuous models are studied in \citet{bengio1994learning,pascanu2013difficulty,chen2018neuralode,bai2019deq}.

\section{Conclusion}
\label{sec:conclusion}

Stochastic cell updates play different roles during learning and execution. Under one fixed-rate persist recipe, asynchronous training yields more final quality passes, while every passing asynchronous model can be evaluated deterministically. In a linear lattice, the same random mask both changes mean propagation and injects variance. After reconstruction, the states used during training distinguish retention from repair. These results make update schedule, state exposure, and task evaluation separate design choices for Growing NCA.

\subsubsection*{Ethics statement}
This work studies small generative lattice models on two synthetic images. It involves no human subjects or personal data.

\subsubsection*{Reproducibility statement}
The appendix provides the architecture, hyperparameters, evaluation definitions, original comparison criteria, model-level measurements, and proof details. It distinguishes prespecified comparisons from exploratory controls and states the update tapes and statistical units.

\subsubsection*{AI use disclosure}
Large language models were used for language polishing. The authors reviewed and take responsibility for the manuscript.

\bibliography{references}
\bibliographystyle{iclr2027_conference}
\clearpage
\appendix
\section{Study design and interpretation}
\label{app:discussion}

\textbf{Controlled comparisons.} The update-mode comparison fixes the persist recipe and optimizer while varying the training update probability. The recipe comparison fixes the network and base optimization settings while changing the states used for training. All thirty recipe models pass a common short-horizon quality threshold.

\textbf{Interpretation of measurements.} The mean and second-moment calculations concern a fixed linear operator. The frozen Jacobian is measured at one state with its living masks held fixed. Finite-time perturbations follow a changing trajectory, while long-horizon retention and damage recovery directly evaluate task outcomes. These measurements describe different properties.

\textbf{Protocol status.} The update-mode, recipe, perturbation, recovery, and lattice comparisons follow their recorded protocols. Final-state inspection, lower-rate retraining, and living-mask removal are exploratory controls. Appendix~\ref{app:registry} records the original endpoint definitions and outcomes, including the failed spatial-ratio endpoint and the marginal lattice cases.

\section{Proofs}
\label{app:proofs}

The linear model is $x_{t+1}=x_t+M_tAx_t$, where $A$ is fixed and $M_t=\operatorname{diag}(m_u)$. The variables $m_u$ are independent Bernoulli draws with parameter $\alpha$, independent across cells and time and independent of $x_t$. The model does not include the living masks of Eq.~\eqref{eq:nca}.

\subsection{Proof of Proposition~\ref{prop:disk}}

Independence gives $\mathbb E[M_tAx_t]=\alpha A\mathbb E[x_t]$, so
\[
\mathbb E[x_{t+1}]=(I+\alpha A)\mathbb E[x_t].
\]
A finite-dimensional linear iteration converges to zero from every initial state exactly when all its eigenvalues have modulus below one. Substituting $a=\mu-1$ yields
\[
|1+\alpha(\mu-1)|<1
\iff
\left|\mu-\left(1-\frac1\alpha\right)\right|<\frac1\alpha.
\]
The disk has center $1-1/\alpha$ and radius $1/\alpha$. Its rightmost point is $+1$, which is a boundary point rather than a strictly stable multiplier. As $\alpha$ decreases, the disks approach the half-plane $\operatorname{Re}\mu<1$.

For $\mu=e^{i\theta}$, direct expansion gives
\[
|(1-\alpha)+\alpha e^{i\theta}|^2
=(1-\alpha)^2+2\alpha(1-\alpha)\cos\theta+\alpha^2
=1-2\alpha(1-\alpha)(1-\cos\theta).
\]
This is strictly below one for $0<\alpha<1$ and $\theta\notin2\pi\mathbb Z$. At $\alpha=1/2$ and $\theta=\pi$, it is zero. On the other hand, if $\operatorname{Re}\mu\geq1$, then
\[
\operatorname{Re}[1+\alpha(\mu-1)]
=1+\alpha(\operatorname{Re}\mu-1)\geq1,
\]
so the modulus cannot be below one. This also shows that any multiplier outside the unit disk admitted by $D_\alpha$ must have real part below one.
\hfill$\square$

\subsection{Proof of Theorem~\ref{thm:ms}}

Write $M_t=\alpha I+D_t$, with centered independent diagonal entries of variance $\alpha(1-\alpha)$. Then
\[
x_{t+1}=(I+\alpha A)x_t+D_tAx_t.
\]
For the scalar translation-invariant operator on the $N^d$ torus, the unitary Fourier transform diagonalizes the first term,
\[
\widehat{(I+\alpha A)x_t}(\omega)
=(1+\alpha\hat a(\omega))\hat x_t(\omega).
\]
Condition on $x_t$ and put $v=Ax_t$. The cross term between the mean update and the noise term has zero expectation. If $D_t=\operatorname{diag}(d_u)$, independence of its entries gives
\begin{align*}
\mathbb E\!\left[|\widehat{D_tv}(\omega)|^2\,\middle|\,x_t\right]
&=\frac1{N^d}\sum_{u,u'}
 e^{-i\omega\cdot(u-u')}\mathbb E[d_ud_{u'}]v_uv_{u'}^*\\
&=\frac{\alpha(1-\alpha)}{N^d}\sum_u|v_u|^2\\
&=\frac{\alpha(1-\alpha)}{N^d}\sum_{\omega'}|\hat v(\omega')|^2.
\end{align*}
The last step is Parseval's identity. Taking expectation over $x_t$ and using $\hat v(\omega')=\hat a(\omega')\hat x_t(\omega')$ proves Eq.~\eqref{eq:ms}. No independence assumption on different Fourier coefficients of $x_t$ is needed.

Consider the non-neutral modes $\Omega_*=\{\omega\,\mid\,\hat a(\omega)\neq0\}$. Their power vector evolves autonomously under the nonnegative matrix
\[
\mathcal T=\operatorname{diag}(d_\alpha)+\mathbf1c^\top,
\qquad
c(\omega)=\frac{\alpha(1-\alpha)|\hat a(\omega)|^2}{N^d}.
\]
Mean-square decay requires $\rho(\mathcal T)<1$. Since the injected power is nonnegative, $d_\alpha(\omega)<1$ on $\Omega_*$ is necessary. Under that condition, define the positive vector $q(\omega)=[1-d_\alpha(\omega)]^{-1}$. Its update satisfies
\[
(\mathcal Tq)(\omega)
=d_\alpha(\omega)q(\omega)+S,
\qquad
(\mathcal Tq)(\omega)-q(\omega)=S-1.
\]
If $S<1$, then $\mathcal Tq<q$ entrywise, which implies $\rho(\mathcal T)<1$. If $S\geq1$, then $\mathcal Tq\geq q$, which implies $\rho(\mathcal T)\geq1$. These statements follow from the positive-vector characterization and monotonicity of the spectral radius for nonnegative matrices \citep{horn2013matrix}. Hence Eqs.~\eqref{eq:ms} and~\eqref{eq:criterion} give the necessary and sufficient condition for mean-square asymptotic decay.

For a neutral mode, $\hat a(\omega)=0$ and $d_\alpha(\omega)=1$. Its power receives the common injection but contributes nothing to its size. When the non-neutral modes decay, their injection is summable, so the neutral power remains bounded but need not tend to zero. This is why the theorem excludes neutral modes from the decay claim.
\hfill$\square$

\subsection{Locality under spatial and temporal extension}

\begin{proposition}[Spatial locality and temporal composition]
\label{prop:locality}
Let a translation-equivariant local rule on $\mathbb Z^2$ have dependence radius $r$ in the maximum norm. After $T$ steps, a cell's state depends only on initial states within distance $rT$ and, for random updates, the masks in the corresponding space-time dependency cone. Jointly translating the initial state and mask field translates the output. Independent identically distributed masks therefore give translation equivariance in distribution.

On a square grid of side $L$, at most
\[
1-(1-2rT/L)_+^2
\]
of the cells can have a dependency cone intersecting the boundary. Matching interior cones give matching states under coupled masks, or the same distribution under identically distributed masks. Deterministic time extension composes the update map repeatedly. For a random mask tape $\xi$, with $\Phi_T^\xi$ denoting its first $T$ steps and $\sigma^T\xi$ the shifted tape,
\[
\Phi_{T+S}^\xi
=\Phi_S^{\sigma^T\xi}\circ\Phi_T^\xi.
\]
\end{proposition}

\emph{Proof.} One step enlarges the dependence neighborhood by at most $r$. Induction gives the radius-$rT$ cone, including the masks used in its intermediate updates. Equivariance follows by translating every input to the same shared local rule. Counting cells at least $rT$ from all four boundaries gives the fraction bound. The composition identity follows by splitting the consecutive updates at step $T$.
\hfill$\square$

For Eq.~\eqref{eq:nca}, the complete update has dependence radius at most two. The post-update living mask reads neighboring candidate states, each computed from a radius-one perception neighborhood. The bound therefore uses $r=2$. At $T=96$, $2rT/L=\ConeRatioSmall$ for $L=72$, and the boundary-fraction bound equals $\ConeFracLarge$ at $L=288$. Spatial performance at these sizes is measured directly by the tiling evaluation in App.~\ref{app:tiling}; the proposition identifies the dependency structure and its distinction from additional temporal composition.

\section{Measurement definitions}
\label{app:instruments}

The model is the statistical unit, with five training seeds per recipe and target. Confidence intervals use $10^4$ bootstrap resamples of models within each target and group. Contrasts are evaluated separately on flower and heart. Group intervals are displayed to three decimal places.

\textbf{Targets and initialization.} Flower and heart are $40\times40$ RGBA target images placed in $72\times72$ grids. Generation starts from the standard single-cell seed used by the implementation. The four visible channels and twelve hidden channels evolve together. The loss measures RGBA reconstruction rather than similarity between the white-background illustrations in the figures.

\textbf{Quality test.} A model passes when its target RGBA MSE at $T=96$ is below $\QualityBar$. The update-mode comparison evaluates each model with its training update probability. Synchronous models use deterministic updates; asynchronous models use the quality tape with seed 42. All recipe models pass the same reconstruction threshold. Their stochastic $T=96$ errors are listed individually in Table~\ref{tab:permodelb}.

\textbf{Mode-swap difference.} The evaluator records the signed change
\[
D_{\mathrm{swap}}
=E_{1024}^{\mathrm{unmatched}}-E_{1024}^{\mathrm{matched}},
\]
where the stochastic error is averaged over three tapes. The main text reports median absolute changes among quality-passing models. The signed values are retained in Table~\ref{tab:permodela}, so an improvement after switching remains visible. The recorded medians coincide with the absolute-change medians for the groups reported in the main text. The preregistered asymmetry test uses the signed differences as specified in its protocol.

\textbf{Long-horizon outcomes.} Checkpoints are recorded at $T\in\{96,256,512,1024,2048,4096\}$. A rollout with no living cells is labeled dead. A non-finite state or $\max|x|>10^3$ is labeled diverged. For surviving finite rollouts, final MSE above $0.15$ is labeled off-target; the remainder is labeled bounded in the archived evaluator. We call this last category \emph{retained} in the tables to distinguish it from numerical boundedness alone. The $0.15$ long-horizon threshold differs from the $0.02$ reconstruction threshold. Empty and divergent rollouts can contain propagated sentinel MSE values, which are not treated as ordinary measured errors.

\textbf{Finite-time perturbation growth.} The recipe-study reference state is reached at $T=96$ using update tape seed $7100+\mathrm{model\ index}$. Perturbations use Gaussian seeds 41, 42, and 43 with amplitude $10^{-3}$ per state component. They are restricted to cells whose own alpha channel exceeds $0.1$, rather than the neighborhood-pooled living mask used by the update rule. The perturbed and unperturbed trajectories share steps 96 through 159 of the reference tape. We compute Eq.~\eqref{eq:pert} over those 64 steps, without renormalizing the perturbation, and average the three values.

States and norms use float32, followed by a scalar logarithm. A negative value denotes shrinkage of the sampled finite-amplitude perturbation over 64 steps. The perturbation is not renormalized or aligned to a maximal-growth direction, so the quantity differs from a maximum asymptotic Lyapunov exponent. Long-horizon and damage measurements evaluate the associated task behavior.

\textbf{Frozen-state Jacobian.} For an evaluation state $x_*$, define
\[
y_*=x_*+f_\theta(K*x_*),\qquad
g_*=g(x_*)\odot g(y_*).
\]
The differentiated map is $F_*(x)=g_*\odot[x+f_\theta(K*x)]$, holding the product of both living masks fixed. The recipe study uses its stochastic state at $T=96$; the factorial uses the matched-mode state at $T=512$. The network and state are converted to float64 for forward-mode Jacobian-vector products. Arnoldi requests 32 eigenvalues by modulus with tolerance $10^{-8}$ and a maximum of 3,000 iterations. The strict disk geometry in Fig.~\ref{fig:spectra} uses $|\mu+1|<2$. The preregistered factorial placement endpoint instead uses its recorded tolerance, $|\mu+1|<2.05$.

\textbf{Oscillation measurement.} The secondary factorial endpoint uses deterministic rollouts in the retained class. It records the global living-cell mean RGBA trace from steps 2,048 through 4,096, subtracts each channel's temporal mean, and takes the square root of the mean channel variance. This temporal amplitude is distinct from the spatial spectra used in the exploratory failure-state probe.

\textbf{Damage recovery.} At $T=96$, circular damage with severity $0.30$ is applied to a mature state. Clean and damaged copies share the same update masks for a recovery window of 80 steps. Let $E_{\mathrm{pre}}$ be the error before damage, $E_d$ the error immediately afterward, and $E_{\mathrm{final}}$ the recorded damaged-copy error at index $H-1$ for $H=80$. The recovery score is
\[
R=100\frac{E_d-E_{\mathrm{final}}}{E_d-E_{\mathrm{pre}}}.
\]
Full recovery gives $R=100$, no improvement gives $R=0$, and a larger error than immediately after damage gives $R<0$. Table~\ref{tab:regimes} reports the mean recovery across models within each recipe and target.

\textbf{Tiling.} Worlds of side $72m$ contain $m^2$ seeds at tile centers for $m\in\{1,2,4\}$. Each model runs for 96 stochastic steps with a frozen tape. We first take the median RGBA error over its tiles, then the median over models within a recipe and target. The ratio endpoint instead forms each model's $m=4$ to $m=1$ error ratio before taking the recipe median.

\section{Prespecified comparisons and observations}
\label{app:registry}

The comparison protocol was fixed before the corresponding evaluations. It specifies the seven quantitative criteria in Table~\ref{tab:endpoints} and a training-reliability analysis when at least six of ten synchronous models fail the quality test. The observed count of seven selects that analysis. Final-state inspection, lower-rate retraining, and living-mask removal are exploratory interventions.

The table keeps each original criterion alongside its observed value. The lattice criterion concerns all $\LatticeCells$ configurations, with $\LatticeMatched/\LatticeCells$ agreement. The $\LatticeNonMarginal/\LatticeNonMarginal$ non-marginal comparison is a subsequent diagnostic of that result, not a prespecified exclusion. Appendix~\ref{app:lattice} explains the boundary classification. The recipe contrasts and spatial comparison likewise retain their per-target intervals and error ratios.

\begin{table}[!t]
\centering
\caption{Original comparison criteria and observed quantities. Synchronous and asynchronous refer to training. The complete lattice comparison and subsequent non-marginal diagnostic are listed separately within the same row.}
\label{tab:endpoints}
\small
\setlength{\tabcolsep}{4pt}
\renewcommand{\arraystretch}{1.17}
\begin{tabularx}{\textwidth}{@{}>{\raggedright\arraybackslash}p{0.17\textwidth} >{\raggedright\arraybackslash}X >{\raggedright\arraybackslash}p{0.31\textwidth}@{}}
\toprule
Comparison & Original criterion & Observed quantities \\
\midrule
Mode-swap asymmetry & Asynchronous-to-synchronous swap ratio above 3, the same direction on both targets, and ratio CI excluding 1. & CI $\SwapRatioInterval$; directions differ between targets. \\
\addlinespace[4pt]
Spectral placement & At least 6 asynchronous models with $\rho>1.02$ and all violating multipliers within $|\mu+1|<2.05$; at most 2 qualifying synchronous survivors. & $\QualifiedAsync$ asynchronous models and $\QualifiedSync$ synchronous survivor qualify. \\
\addlinespace[4pt]
Temporal oscillation & Synchronous amplitude above twice the asynchronous amplitude on both targets. & The ordering reverses on flower. \\
\addlinespace[4pt]
Perturbation contrasts & Grow minus persist and grow minus regenerate intervals strictly above zero on both targets. & $\PositiveContrasts/4$ intervals exclude zero; heart grow minus regenerate is $\CIgrH$. \\
\addlinespace[4pt]
Long-horizon ordering & Grow median MSE at $T=4096$ above persist and regenerate on both targets. & Grow $\GrowMseLongF/\GrowMseLongH$; other groups at most $\OtherMseLongMax$. \\
\addlinespace[4pt]
Spatial error ratio & Median per-model error ratio from $m=1$ to $m=4$ below 1.2 for both recipes. & Persist $\TileRatioPersist$; regenerate $\TileRatioRegen$. \\
\addlinespace[4pt]
Lattice agreement & Agreement with simulation on all 27 configurations and at least one mean-only error. & $\LatticeMatched/\LatticeCells$ overall; subsequent non-marginal diagnostic $\LatticeNonMarginal/\LatticeNonMarginal$, with $\LatticeAnnWrong$ mean-only errors. \\
\bottomrule
\end{tabularx}

\end{table}

The swap and spectral comparisons use quality-passing models as specified by their definitions. The temporal oscillation comparison uses retained deterministic rollouts and the global RGBA trace, rather than a spatial Fourier measurement. The recovery and living-mask results answer complementary behavioral questions and are not additional tests of these original criteria.

\section{Linear lattice validation}
\label{app:lattice}

The archived experiment uses a $32\times32$ periodic scalar lattice with the five-point Laplacian,
\[
\hat a(\omega)=c(2\cos\omega_x+2\cos\omega_y-4)\in[-8c,0].
\]
Its 27 configurations combine $c\in\{0.20,0.25,0.30,0.35,0.40,0.45,0.50,0.60,0.70\}$ with $\alpha\in\{0.25,0.5,1\}$. Each configuration uses mask seeds 1, 2, and 3 and runs for up to 8,000 steps.

The constant Fourier mode is neutral and is removed by centering after every simulated update. The simulation starts with a centered Gaussian field of scale $10^{-3}$. It stops early at maximum absolute state below $10^{-14}$ or above $10^{10}$. If neither threshold is reached, it labels a run as decay when its final maximum norm is smaller than its initial norm. A configuration is simulated as stable when all three runs are labeled decay.

This finite-time classification differs from asymptotic decay at a boundary. At $c=0.25,\alpha=1$, the checkerboard multiplier is exactly $-1$ and preserves its amplitude, yet the final norm is lower than the initial mixed-mode norm. The simulation therefore labels decay even though the strict theoretical criterion does not. The raw all-configuration result is $\LatticeMatched/\LatticeCells$ agreement. A second configuration, $c=0.50,\alpha=0.50$, also lies on the mean-stability boundary. Reporting both marginal configurations separately leaves $\LatticeNonMarginal/\LatticeNonMarginal$ agreement and $\LatticeAnnWrong$ mean-only misclassifications away from the boundary. This restricted comparison is a diagnostic of the recorded result, as distinguished from the original endpoint in App.~\ref{app:registry}.

\begin{table}[!t]
\centering
\caption{All 27 lattice configurations arranged by update probability. Each block reports the mean-only prediction, second-moment prediction, simulated label, and $S$. A check denotes decay and a cross denotes no decay. A dagger marks an exactly marginal mean multiplier. Bold mean-only predictions identify the four non-marginal disagreements with simulation. An infinite $S$ denotes failure of the prerequisite $d_\alpha<1$ in the implementation.}
\label{tab:lattice}
\small
% Generated from immutable results.
{\fontsize{8}{10}\selectfont
\setlength{\tabcolsep}{4pt}
\renewcommand{\arraystretch}{1.18}
\begin{tabular*}{\textwidth}{@{\extracolsep{\fill}}ccccccccccccc@{}}
\toprule
$c$ & \multicolumn{4}{c}{$\alpha=0.25$} & \multicolumn{4}{c}{$\alpha=0.5$} & \multicolumn{4}{c}{$\alpha=1.0$} \\
\cmidrule(lr){2-5}\cmidrule(lr){6-9}\cmidrule(lr){10-13}
 & Mean & MS & Sim. & $S$ & Mean & MS & Sim. & $S$ & Mean & MS & Sim. & $S$ \\
\midrule
$0.20$ & $\checkmark$ & $\checkmark$ & $\checkmark$ & $0.34$ & $\checkmark$ & $\checkmark$ & $\checkmark$ & $0.27$ & $\checkmark$ & $\checkmark$ & $\checkmark$ & $0.00$ \\
$0.25$ & $\checkmark$ & $\checkmark$ & $\checkmark$ & $0.45$ & $\checkmark$ & $\checkmark$ & $\checkmark$ & $0.37$ & $\times$ & $\times$ & $\checkmark^\dagger$ & $\infty$ \\
$0.30$ & $\checkmark$ & $\checkmark$ & $\checkmark$ & $0.56$ & $\checkmark$ & $\checkmark$ & $\checkmark$ & $0.50$ & $\times$ & $\times$ & $\times$ & $\infty$ \\
$0.35$ & $\checkmark$ & $\checkmark$ & $\checkmark$ & $0.68$ & $\checkmark$ & $\checkmark$ & $\checkmark$ & $0.67$ & $\times$ & $\times$ & $\times$ & $\infty$ \\
$0.40$ & $\checkmark$ & $\checkmark$ & $\checkmark$ & $0.81$ & $\checkmark$ & $\checkmark$ & $\checkmark$ & $0.92$ & $\times$ & $\times$ & $\times$ & $\infty$ \\
$0.45$ & $\checkmark$ & $\checkmark$ & $\checkmark$ & $0.96$ & $\boldsymbol{\checkmark}$ & $\times$ & $\times$ & $1.35$ & $\times$ & $\times$ & $\times$ & $\infty$ \\
$0.50$ & $\boldsymbol{\checkmark}$ & $\times$ & $\times$ & $1.12$ & $\times$ & $\times$ & $\times^\dagger$ & $\infty$ & $\times$ & $\times$ & $\times$ & $\infty$ \\
$0.60$ & $\boldsymbol{\checkmark}$ & $\times$ & $\times$ & $1.51$ & $\times$ & $\times$ & $\times$ & $\infty$ & $\times$ & $\times$ & $\times$ & $\infty$ \\
$0.70$ & $\boldsymbol{\checkmark}$ & $\times$ & $\times$ & $2.02$ & $\times$ & $\times$ & $\times$ & $\infty$ & $\times$ & $\times$ & $\times$ & $\infty$ \\
\bottomrule
\end{tabular*}
}

\end{table}

\section{Spatial extension}
\label{app:tiling}

Persist and regenerate are evaluated on worlds of side 72, 144, and 288, containing one, four, and sixteen seeds respectively. Figure~\ref{fig:tiling} reports the per-target median tile errors. Every group median remains at the $10^{-4}$ scale and at least $\TileMedMargin$ times below the quality threshold. The largest per-model median tile error is $\TileMaxAbs$.

The median per-model error ratio is $\TileRatioPersist$ for persist and $\TileRatioRegen$ for regenerate. Relative to the original threshold of 1.2, persist is above the threshold and regenerate below it. The absolute errors and ratios therefore describe different aspects of spatial extension. The former measures the achieved image quality on larger worlds; the latter measures its change relative to each model's single-tile error.

\begin{figure}[!t]
\centering
\includegraphics[width=0.76\textwidth]{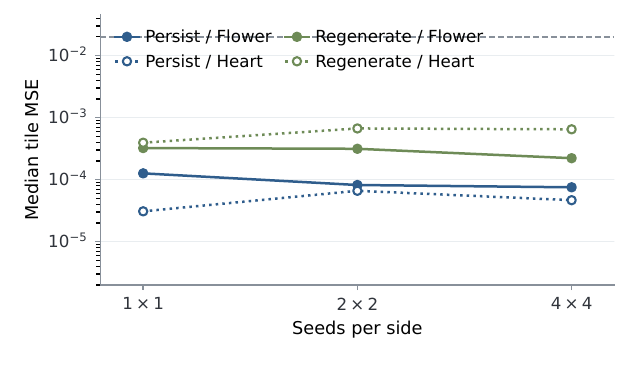}
\caption{\textbf{Absolute error under spatial extension.} Per-target median tile MSE for one, four, and sixteen simultaneous seeds at $T=96$. The dashed line is the quality threshold. The ratio endpoint and the absolute-error observation are distinct statements.}
\label{fig:tiling}
\end{figure}

\section{Exploratory interventions}
\label{app:probes}

\subsection{Final states of failed synchronous models}

The probe runs each factorial checkpoint deterministically from the seed for 96 steps and computes channel-averaged two-dimensional spectra of its RGBA state and update field. The high-frequency fraction is the share of non-DC power with $\max(|\omega_x|,|\omega_y|)\geq\pi/2$. The checkerboard fraction uses frequencies within maximum-norm distance $\pi/4$ of $(\pi,\pi)$.

Six quality-failing synchronous models are empty and have zero non-DC power in both objects. The seventh, heart seed three, is alive with deterministic MSE approximately $\FailureAliveMse$. Its high-frequency fractions are $\FailureAliveStateHF$ for the state and $\FailureAliveUpdateHF$ for the update. This final-state measurement identifies extinction as the most common terminal outcome, complementing the loss traces that follow training over time.

\subsection{Lower learning rates}

Flower seeds zero and two fail under synchronous training at $2\cdot10^{-3}$. Retraining each at $10^{-3}$ and $5\cdot10^{-4}$ gives four additional runs, with the other settings unchanged. All four pass the quality test. Figure~\ref{fig:lrprobe} shows the trajectories for these two successful alternatives to the original learning rate.

\begin{figure}[!t]
\centering
\includegraphics[width=0.8\textwidth]{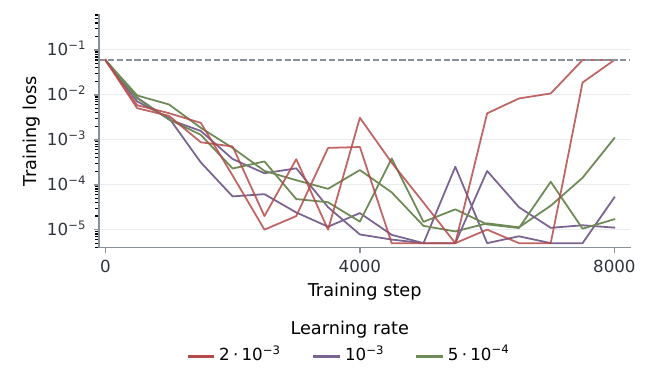}
\caption{\textbf{Lower-rate synchronous retraining.} Two failed flower configurations are retrained at each of two lower learning rates. All four runs pass the quality test. The traces are exploratory and use the same training budget as the corresponding main runs. Zero-valued logged losses are drawn at $5\cdot10^{-6}$ on the logarithmic axis.}
\label{fig:lrprobe}
\end{figure}

\subsection{Single-run checkpoint replay}
\label{app:checkpoint-probe}

We replayed one failed synchronous flower configuration using the original persist recipe, initialization seed, constant learning rate $2\cdot10^{-3}$, and 8,000-step budget. The first checkpoint saved when the logged minibatch loss fell below $10^{-3}$ occurred at step 1,053. From the seed, this checkpoint has $T=96$ MSE $0.002099$ and $T=4096$ MSE $0.015803$, with 552 living cells. The final checkpoint becomes empty and dies by step 47; its empty-canvas MSE is $0.059658$. This one-run replay confirms that a temporarily usable checkpoint can occur in a failed configuration. It does not estimate the frequency of such checkpoints, identify why later updates change the outcome, or establish a training mechanism. The saved weights, training log, evaluation images, and protocol record are in the checkpoint-probe artifact directory.

\subsection{Removing the living masks}
\label{app:gateoff}

For each of the thirty recipe models, we replay the exact update-mask tape from its original long rollout. The three conditions retain both living masks throughout, remove both after $T=96$, or omit both from the seed. The gate-live condition reproduces all thirty original outcome labels.

Neither mask-removal condition produces a divergent or dead run. At $T=\HorizonLong$, the largest $\max|x|$ is approximately $\GateOffMaxAbs$ after removal at maturity and $\GateOffSeedMaxAbs$ after removal from the seed. For the mature-state intervention, the median across models is $\GateOffMedAbs$. Among models with a real multiplier above one, the largest final magnitude is $\GateOffCarrierMaxAbs$. Table~\ref{tab:gateoff} retains the per-group ranges and outcome counts.

Removing the masks at maturity changes two outcomes, both regenerate-heart models that become off-target. The remaining 28 outcomes are unchanged. Living-cell counts increase by more than 20\% for $\GateOffLeakMat$ models in the mature-state intervention and $\GateOffLeakSeed$ in the seed intervention. The largest increase is approximately $\GateOffLeakFactor$ times the gate-live count. These finite-horizon interventions distinguish the masks' spatial and pattern-retention effects from numerical boundedness.

\begin{table}[!t]
\centering
\caption{Living-mask ablation with matched random update tapes. Each group contains five models. Outcome counts give retained, off-target, and diverged or dead runs in that order. State magnitudes are ranges across the five models at $T=\HorizonLong$.}
\label{tab:gateoff}
\small
% Generated from immutable results.
{\fontsize{8}{10}\selectfont
\setlength{\tabcolsep}{4pt}
\renewcommand{\arraystretch}{1.18}
\begin{tabular*}{\textwidth}{@{\extracolsep{\fill}}llcccccc@{}}
\toprule
 &  & \multicolumn{2}{c}{Masks active} & \multicolumn{2}{c}{Removed at $T=96$} & \multicolumn{2}{c}{Absent from seed} \\
\cmidrule(lr){3-4}\cmidrule(lr){5-6}\cmidrule(lr){7-8}
Recipe & Target & Outcomes & $\max|x|$ & Outcomes & $\max|x|$ & Outcomes & $\max|x|$ \\
\midrule
Grow & Flower & 1/4/0 & $[1.02,3.97]$ & 1/4/0 & $[1.02,16.54]$ & 1/4/0 & $[1.03,24.35]$ \\
 & Heart & 1/4/0 & $[1.07,8.55]$ & 1/4/0 & $[1.10,9.62]$ & 1/4/0 & $[1.12,10.17]$ \\
\addlinespace[3pt]
Persist & Flower & 5/0/0 & $[1.01,1.09]$ & 5/0/0 & $[1.01,1.09]$ & 5/0/0 & $[1.01,1.02]$ \\
 & Heart & 5/0/0 & $[0.99,1.02]$ & 5/0/0 & $[0.99,1.05]$ & 5/0/0 & $[0.99,1.04]$ \\
\addlinespace[3pt]
Regenerate & Flower & 5/0/0 & $[1.02,1.39]$ & 5/0/0 & $[1.02,1.31]$ & 5/0/0 & $[1.02,1.14]$ \\
 & Heart & 5/0/0 & $[1.00,1.35]$ & 3/2/0 & $[1.00,1.21]$ & 3/2/0 & $[1.00,1.46]$ \\
\bottomrule
\end{tabular*}
}

\end{table}

\section{Frozen spectral measurements}
\label{app:eigs}

The synchronous frozen-living-mask Jacobian has dimension 82,944. Forward-mode Jacobian-vector products allow Arnoldi iteration without forming the matrix, using the settings in App.~\ref{app:instruments}. Each model has a returned window of 32 eigenvalues, with no recorded solver error. The stated tolerance is the solver's stopping tolerance; the reported quantities are the returned eigenvalues and their positions relative to the stability disks.

Figure~\ref{fig:spectra} shows the returned eigenvalue locations and their relation to the measured long-horizon outcomes. It is included as a diagnostic comparison, not as a spectrum-based predictor of nonlinear fate.

\begin{figure}[!t]
\centering
\includegraphics[width=\textwidth]{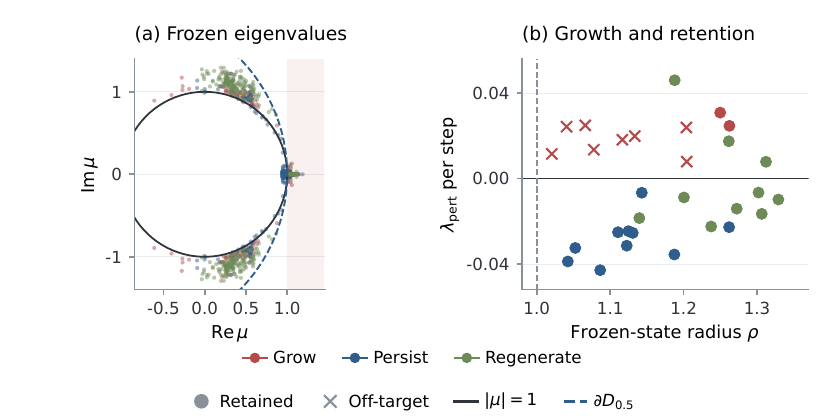}
\caption{\textbf{A frozen local expansion measurement does not determine long-horizon task outcome.} (a) The largest-modulus eigenvalue returned for each trained model at its frozen mature state. (b) Frozen-state radius and finite-time perturbation growth, with markers showing long-horizon outcome. The same condition $\rho>1$ occurs in retained and off-target models.}
\label{fig:spectra}
\end{figure}

Among the $\NViol$ returned eigenvalues with modulus above one, $\NViolInDisk$ lie inside the strict $D_{0.5}$ disk, $\NRealUnforgivable$ are real and exceed one, and $\NViolRePosCplx$ are complex with real part above one. Regenerate contributes $\ViolDiskRegen$ values in the disk. Grow contributes $\ViolReGrow$ values with real part above one, compared with $\ViolReRegen$ for regenerate. Figure~\ref{fig:spectralcounts} displays the full recipe breakdown.

\begin{figure}[!t]
\centering
\includegraphics[width=0.78\textwidth]{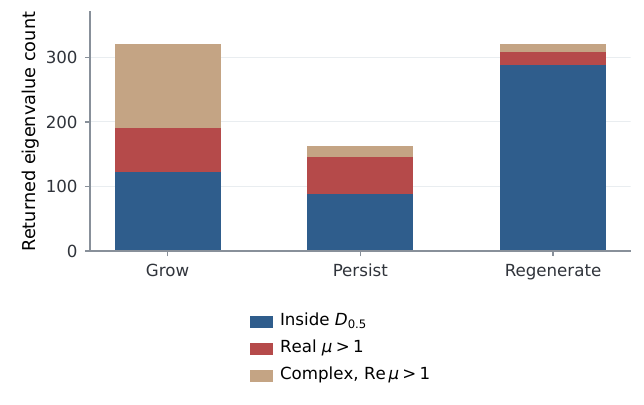}
\caption{\textbf{Returned multipliers outside the unit disk.} Counts are grouped by recipe and position relative to $D_{0.5}$ and $\operatorname{Re}\mu=1$. They refer to the largest-modulus 32-eigenvalue window per model, not the complete spectrum.}
\label{fig:spectralcounts}
\end{figure}

For $\NEigSaturated$ models, all $\EigWindow$ returned eigenvalues lie outside the unit disk. The total count is therefore a lower bound on the number of such multipliers in the full Jacobians. Real multipliers above one occur on all $\GrowCarriers$ grow models, all $\PersistCarriers$ persist models, and $\RegenCarriers$ regenerate models. Eight of those grow models become off-target, while all fourteen persist and regenerate models retain the pattern. These counts preserve the distinction between a frozen local expansion and an observed long-horizon outcome.

\section{Individual model results}
\label{app:permodel}

The following tables retain the model-level measurements behind the summaries. Table~\ref{tab:permodelb} contains the thirty asynchronous recipe models. Table~\ref{tab:permodela} contains the twenty factorial entries, ten of which reuse the persist models from the recipe study. The two tables therefore describe forty distinct base models.

\begin{table}[!t]
\centering
\caption{The thirty recipe models. Reconstruction errors, finite-time perturbation growth, largest returned frozen multiplier, and stochastic long-horizon outcome are listed for every training seed. Retained is the archived bounded outcome class. Outcome labels use unrounded measurements.}
\label{tab:permodelb}
\small
% Generated from immutable results.
{\fontsize{8.5}{10.5}\selectfont
\setlength{\tabcolsep}{4pt}
\renewcommand{\arraystretch}{1.18}
\begin{tabular*}{\textwidth}{@{\extracolsep{\fill}}llcccccl@{}}
\toprule
Recipe & Target & Seed & MSE@96 & $\lambda_{\mathrm{pert}}$ & $\rho$ & MSE@4096 & Outcome \\
\midrule
Grow & Flower & 0 & $6.6\times 10^{-5}$ & $+0.0198$ & $1.134$ & $0.39$ & Off-target \\
 &  & 1 & $1.4\times 10^{-5}$ & $+0.0247$ & $1.262$ & $0.05$ & Retained \\
 &  & 2 & $2.0\times 10^{-3}$ & $+0.0182$ & $1.117$ & $0.24$ & Off-target \\
 &  & 3 & $6.8\times 10^{-5}$ & $+0.0135$ & $1.078$ & $0.26$ & Off-target \\
 &  & 4 & $2.3\times 10^{-5}$ & $+0.0080$ & $1.204$ & $0.35$ & Off-target \\
\addlinespace[3pt]
 & Heart & 0 & $1.9\times 10^{-5}$ & $+0.0239$ & $1.204$ & $0.34$ & Off-target \\
 &  & 1 & $6.7\times 10^{-5}$ & $+0.0308$ & $1.250$ & $0.05$ & Retained \\
 &  & 2 & $5.1\times 10^{-6}$ & $+0.0115$ & $1.020$ & $0.36$ & Off-target \\
 &  & 3 & $1.3\times 10^{-5}$ & $+0.0249$ & $1.066$ & $0.36$ & Off-target \\
 &  & 4 & $1.9\times 10^{-5}$ & $+0.0243$ & $1.041$ & $0.15$ & Off-target \\
\addlinespace[3pt]
Persist & Flower & 0 & $5.6\times 10^{-5}$ & $-0.0314$ & $1.123$ & $2.6\times 10^{-5}$ & Retained \\
 &  & 1 & $1.1\times 10^{-4}$ & $-0.0227$ & $1.262$ & $9.8\times 10^{-3}$ & Retained \\
 &  & 2 & $1.7\times 10^{-3}$ & $-0.0066$ & $1.143$ & $2.2\times 10^{-3}$ & Retained \\
 &  & 3 & $2.1\times 10^{-5}$ & $-0.0324$ & $1.053$ & $1.1\times 10^{-5}$ & Retained \\
 &  & 4 & $2.8\times 10^{-4}$ & $-0.0254$ & $1.131$ & $6.5\times 10^{-4}$ & Retained \\
\addlinespace[3pt]
 & Heart & 0 & $3.3\times 10^{-5}$ & $-0.0388$ & $1.042$ & $2.1\times 10^{-5}$ & Retained \\
 &  & 1 & $7.7\times 10^{-4}$ & $-0.0246$ & $1.125$ & $9.2\times 10^{-4}$ & Retained \\
 &  & 2 & $6.1\times 10^{-5}$ & $-0.0355$ & $1.187$ & $8.5\times 10^{-5}$ & Retained \\
 &  & 3 & $4.0\times 10^{-3}$ & $-0.0251$ & $1.111$ & $4.1\times 10^{-3}$ & Retained \\
 &  & 4 & $4.4\times 10^{-5}$ & $-0.0428$ & $1.086$ & $1.1\times 10^{-5}$ & Retained \\
\addlinespace[3pt]
Regenerate & Flower & 0 & $1.2\times 10^{-3}$ & $-0.0088$ & $1.200$ & $0.04$ & Retained \\
 &  & 1 & $2.7\times 10^{-4}$ & $-0.0141$ & $1.272$ & $1.1\times 10^{-3}$ & Retained \\
 &  & 2 & $6.4\times 10^{-5}$ & $+0.0175$ & $1.261$ & $5.9\times 10^{-5}$ & Retained \\
 &  & 3 & $6.0\times 10^{-4}$ & $-0.0098$ & $1.329$ & $5.2\times 10^{-5}$ & Retained \\
 &  & 4 & $4.3\times 10^{-4}$ & $-0.0224$ & $1.237$ & $7.0\times 10^{-4}$ & Retained \\
\addlinespace[3pt]
 & Heart & 0 & $5.7\times 10^{-4}$ & $-0.0065$ & $1.302$ & $1.3\times 10^{-3}$ & Retained \\
 &  & 1 & $1.7\times 10^{-3}$ & $+0.0078$ & $1.312$ & $1.5\times 10^{-3}$ & Retained \\
 &  & 2 & $3.4\times 10^{-4}$ & $-0.0165$ & $1.306$ & $5.2\times 10^{-4}$ & Retained \\
 &  & 3 & $4.8\times 10^{-4}$ & $-0.0185$ & $1.140$ & $6.8\times 10^{-4}$ & Retained \\
 &  & 4 & $5.0\times 10^{-3}$ & $+0.0460$ & $1.188$ & $0.10$ & Retained \\
\addlinespace[3pt]
\bottomrule
\end{tabular*}
}

\end{table}

\begin{table}[!t]
\centering
\caption{The twenty factorial entries. The last column is the signed mode-swap difference, so a negative value denotes lower error in the unmatched mode. A dagger denotes an empty model for which the frozen quality evaluator records a sentinel rather than an ordinary error value. Its true empty-canvas image loss is identified separately in the training traces and final-state probe.}
\label{tab:permodela}
\small
% Generated from immutable results.
{\fontsize{8.5}{10.5}\selectfont
\setlength{\tabcolsep}{4pt}
\renewcommand{\arraystretch}{1.18}
\begin{tabular*}{\textwidth}{@{\extracolsep{\fill}}llccllc@{}}
\toprule
Training & Target & Seed & Quality MSE & Pass & Det. outcome & \shortstack{Signed swap\\$\Delta\mathrm{MSE}$} \\
\midrule
sync & Flower & 0 & Dead$^\dagger$ & No & Dead & n/a \\
 &  & 1 & $2.1\times 10^{-3}$ & Yes & Retained & $2.3\times 10^{-3}$ \\
 &  & 2 & Dead$^\dagger$ & No & Dead & n/a \\
 &  & 3 & Dead$^\dagger$ & No & Dead & n/a \\
 &  & 4 & Dead$^\dagger$ & No & Dead & n/a \\
\addlinespace[4pt]
 & Heart & 0 & $1.0\times 10^{-3}$ & Yes & Retained & $1.9\times 10^{-3}$ \\
 &  & 1 & Dead$^\dagger$ & No & Dead & n/a \\
 &  & 2 & $1.5\times 10^{-6}$ & Yes & Retained & $3.7\times 10^{-3}$ \\
 &  & 3 & $0.03$ & No & Retained & n/a \\
 &  & 4 & Dead$^\dagger$ & No & Dead & n/a \\
\addlinespace[4pt]
async & Flower & 0 & $1.7\times 10^{-5}$ & Yes & Retained & $1.0\times 10^{-3}$ \\
 &  & 1 & $1.1\times 10^{-4}$ & Yes & Retained & $7.8\times 10^{-3}$ \\
 &  & 2 & $1.9\times 10^{-3}$ & Yes & Retained & $2.5\times 10^{-3}$ \\
 &  & 3 & $2.9\times 10^{-5}$ & Yes & Retained & $-2.9\times 10^{-4}$ \\
 &  & 4 & $8.7\times 10^{-4}$ & Yes & Retained & $3.4\times 10^{-3}$ \\
\addlinespace[4pt]
 & Heart & 0 & $1.5\times 10^{-4}$ & Yes & Retained & $3.3\times 10^{-4}$ \\
 &  & 1 & $6.6\times 10^{-4}$ & Yes & Retained & $1.2\times 10^{-3}$ \\
 &  & 2 & $7.8\times 10^{-5}$ & Yes & Retained & $6.8\times 10^{-4}$ \\
 &  & 3 & $4.6\times 10^{-3}$ & Yes & Retained & $2.2\times 10^{-3}$ \\
 &  & 4 & $5.1\times 10^{-5}$ & Yes & Retained & $1.0\times 10^{-3}$ \\
\addlinespace[4pt]
\bottomrule
\end{tabular*}
}

\end{table}

\section{Reproducibility details}
\label{app:repro}

\textbf{Training and hardware.} The study contains thirty asynchronous recipe models and ten synchronous persist models. Reusing ten persist models in the update-mode comparison gives forty distinct base training runs. Four additional runs test lower learning rates. Training uses PyTorch on NVIDIA H200 hardware with float32 states and computations; frozen-Jacobian measurements use float64.

\textbf{Hyperparameters.} The network uses sixteen state channels, 128 hidden units, a zero-initialized output layer, and alpha threshold 0.1 for both living-mask checks. Training uses Adam at constant rate $2\cdot10^{-3}$, gradient-norm clipping at 1.0, batch size eight, 8,000 steps, and rollout lengths uniformly sampled from 64 through 96. Pool recipes use 1,024 states and seed replacement. Regenerate damages half of each batch with circular masks. The overflow penalty is based on deviation from clipped state values. The synchronous factorial initialization is seeded by $3100+10\,\mathrm{target\ index}+\mathrm{seed}$.

\textbf{Evaluation controls.} The recipe checkpoints are shared across the trajectory, spectral, damage, and spatial measurements. Update-mode swaps keep model weights fixed. Living-mask interventions reuse the original random update tapes, and perturbation comparisons share masks between clean and disturbed states. These controls preserve the intended comparison unit across the reported measurements.

\end{document}